\documentclass[review]{elsarticle}

\usepackage{lineno,hyperref}
\modulolinenumbers[5]

\journal{Applied Soft Computing}

\usepackage[american]{babel}

\usepackage{natbib} 

\usepackage{mathtools} 
\usepackage{booktabs} 
\usepackage{tikz} 

\usepackage{amsfonts}
\usepackage{amssymb}
\usepackage{amsthm}
\usepackage{amsmath}
\usepackage{subfig}
\usepackage{multirow}
\usepackage{bbm}

\newtheorem{theorem}{Theorem}[section]
\newtheorem{remark}{Remark}[section]

\def\R{\mathbb{R}}
\def\confold{\mathrm{conf_{OVA}}}
\def\confour{\mathrm{conf_{\our{}}}}
\def\Pour{P_\mathrm{\our{}}}
\def\conf{\mathrm{conf}}
\def\acc{\mathrm{acc}}

\def \our{SLOVA}

\begin{document}
	
	\begin{frontmatter}
		
		\title{SLOVA: Uncertainty Estimation Using Single Label One-Vs-All Classifier}
		
		\author[jumcsaddress]{Bartosz Wójcik\corref{mycorrespondingauthor}}
		\cortext[mycorrespondingauthor]{Corresponding author}
		\ead{bartwojc@gmail.com}
		
		\author[jupaaaddress]{Jacek Grela}
		
		\author[jumcsaddress]{Marek Śmieja}
		
		\author[jumcsaddress]{Krzysztof Misztal}
		
		\author[jumcsaddress]{Jacek Tabor}
		
		\address[jumcsaddress]{Faculty of Mathematics and Computer Science,\\Jagiellonian University\\Łojasiewicza 6, 30-348 Kraków, Poland}
		
		\address[jupaaaddress]{Faculty of Physics, Astronomy and Applied Computer Science,\\Jagiellonian University\\Łojasiewicza 11, 30-348 Kraków, Poland}
		
		\begin{abstract}
			Deep neural networks present impressive performance, yet they cannot reliably estimate their predictive confidence, limiting their applicability in high-risk domains. We show that applying a multi-label one-vs-all loss reveals classification ambiguity and reduces model overconfidence. The introduced \our{} (Single Label One-Vs-All) model redefines typical one-vs-all predictive probabilities to a single label situation, where only one class is the correct answer. The proposed classifier is confident only if a single class has a high probability and other probabilities are negligible. Unlike the typical softmax function, \our{} naturally detects out-of-distribution samples if the probabilities of all other classes are small. The model is additionally fine-tuned with exponential calibration, which allows us to precisely align the confidence score with model accuracy. We verify our approach on three tasks. First, we demonstrate that \our{} is competitive with the state-of-the-art on in-distribution calibration. Second, the performance of \our{} is robust under dataset shifts. Finally, our approach performs extremely well in the detection of out-of-distribution samples. Consequently, \our{} is a 
			tool that can be used in various applications where uncertainty modeling is required.
		\end{abstract}
		
		
	\end{frontmatter}
	
	
	
	\section{Introduction}\label{sec:intro}
	
	Deep learning models frequently outperform human capabilities in typical computer vision or natural language processing tasks. Despite their impressive performance, neural networks tend to make overconfident decisions, which limits their applicability in high-risk fields such as medical diagnosis \citep{miotto2016deep}, autonomous vehicle control \citep{levinson2011towards}, or the financial and legal sector \citep{berk2017impact}. In other words, deep learning models often cannot correctly quantify predictive uncertainty.
	
	We focus on three aspects that are critical in predictive uncertainty estimation. The first one is the {\em in-distribution calibration}, which says that the probability associated with the predicted class label should reflect its confidence. Second, machine learning models should be {\em robust under dataset shifts} (small data distortion) in terms of accuracy and confidence. Finally, the models should make low confidence predictions on {\em out-of-distribution (OOD) data}. While all of these issues are closely connected, there are not many models that deal with all of them simultaneously.
	
	In this paper, we propose a simple, yet remarkably effective approach, a remedy for all of the described problems. Our idea relies on extending and calibrating the multi-label One-Vs-All (OVA) classifier \citep{chen2019multi} to the single-label case. In contrast to a typical softmax classifier designed to rank classes rather than estimate confidence, OVA models use an individual scoring function for each class. If the model assigns high scores to more than one class or returns low predictions for all classes, it informs us about its uncertainty. Although OVA has already been used for uncertainty estimation \citep{padhy2020revisiting, franchi2020one}, confidence was always defined as the maximal class probability, which results in high confidence even if two or more classes are very likely.
	
	\begin{figure}[h!]
		\centering
		\includegraphics[width=0.65\textwidth]{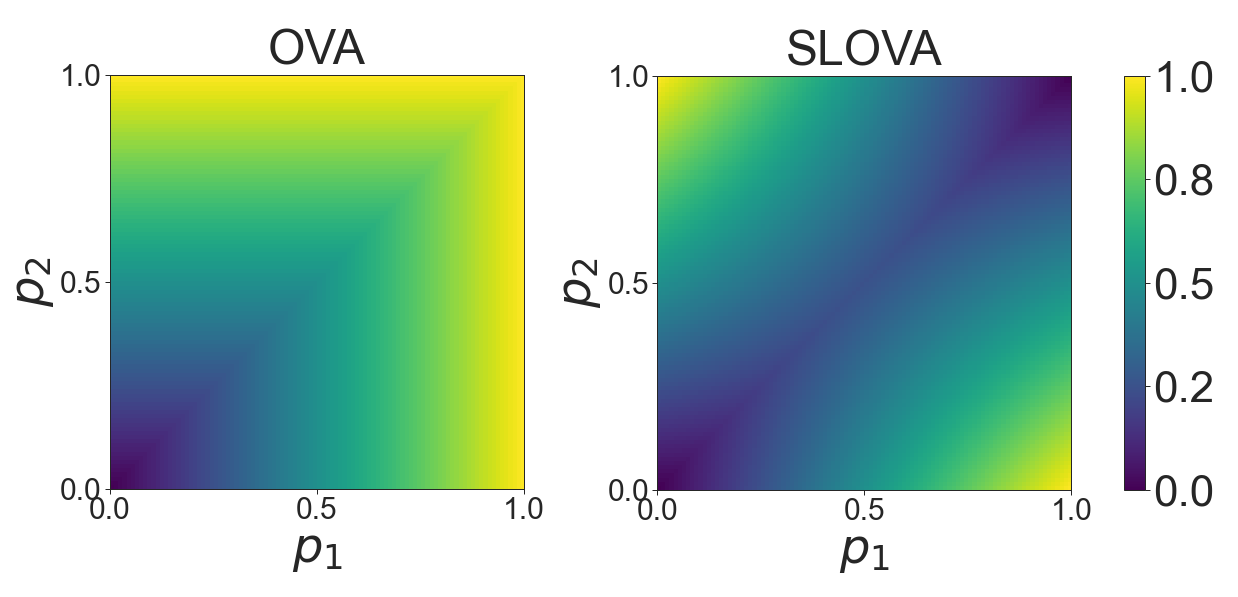} 
		\caption{Confidence score map of OVA and \our{} for a two-class classification problem. The axes $p_1,p_2$ represent the likelihood of the data-point belonging to each class while the colors encode the score. Since OVA connects confidence with the maximal class probability, it may return high confidence if two classes are equally likely. In contrast, \our{} gives high confidence if exactly one class has a high probability.}
		\label{fig:intro}
	\end{figure}
	
	Following this reasoning, we construct a post hoc model, \our{} ({\bf S}ingle-{\bf L}abel-{\bf O}ne-{\bf V}s-{\bf A}ll), which aggregates the predictive probabilities of a pretrained OVA model. It returns high confidence if and only if a single class has a high probability and other probabilities are negligible. See Figure \ref{fig:intro} for a comparison with the typical OVA confidence in the two-class classification problem. We prove that \our{} is guaranteed to return low confidence for out-of-distribution samples; see Theorem \ref{thm:main} and Figure \ref{fig:random_plane_mean}. Finally, to precisely align the confidence score with the model accuracy, we introduce exponential calibration. This transformation is capable of approximating any monotonic function \citep{kammler1979least} and, in consequence, fits perfectly for uncertainty estimation of in-distribution and shifted data.
	
	Since the performance of OVA on a single label problem is comparable to the softmax model, \our{} does not lead to a drop in accuracy. OVA is trained using standard multi-label binary cross-entropy and thus does not require any modification of the neural network architecture. The calculation of confidence scores and the calibration of \our{} is performed as a post-processing step, so the training procedure remains precisely the same; see Figure \ref{fig:intro2} for the model diagram. The implementation and usage of the proposed technique can be done with minimal effort.
	
	\begin{figure}[t]
		\centering
		\includegraphics[width=0.6\textwidth]{ 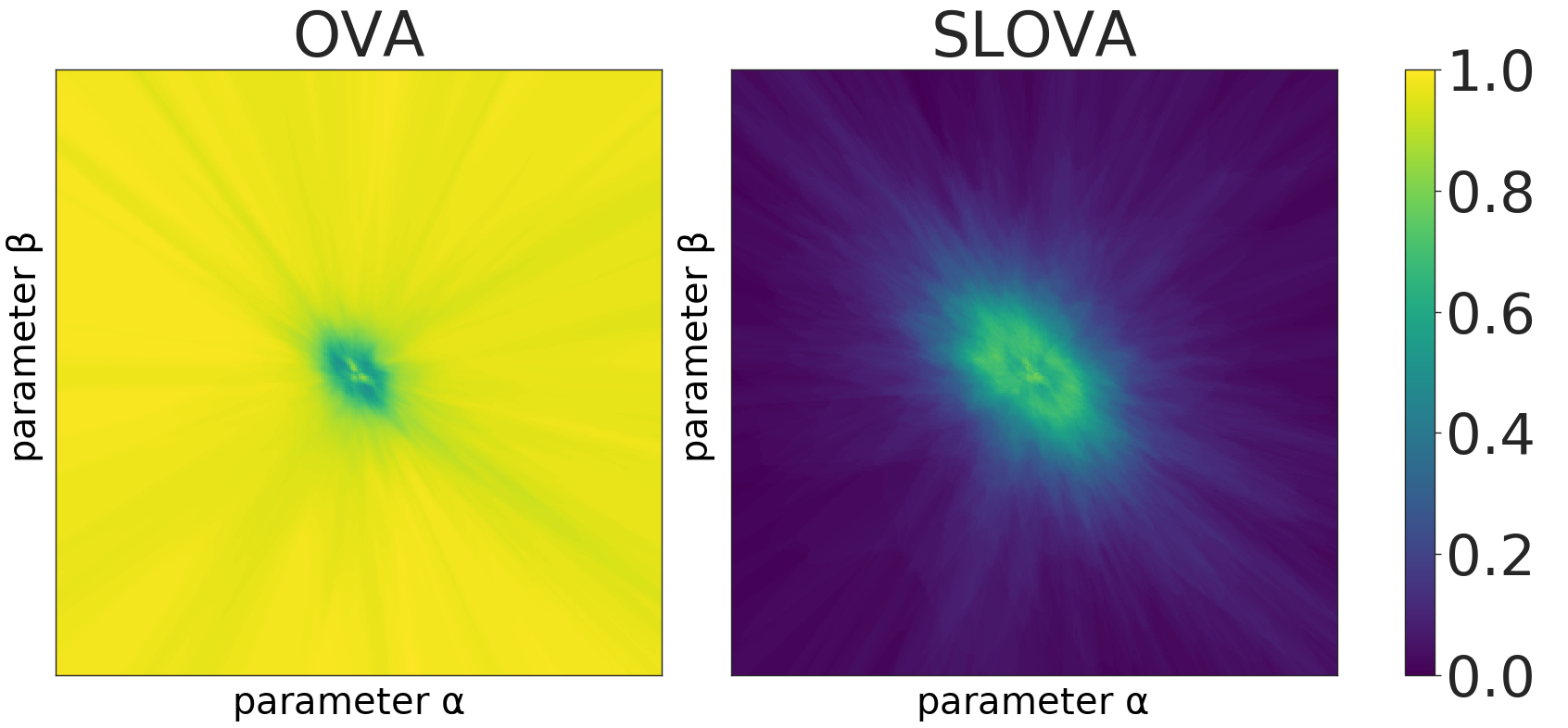}
		\caption{Confidence score across the plane spanned by 3 randomly selected test CIFAR-10 samples, parametrized by two variables $\alpha,\beta$. These parameters are picked so that the plane passes through a single sample point at the center of each plot, resulting in a high confidence score at that point. The color indicates the confidence score of OVA (left) and \our{} (right) averaged over $100$ sample triplets. We see that \our{} returns low confidence on OOD samples (far from the image center), while OVA predictions are close to 1. \label{fig:random_plane_mean}}
	\end{figure}
	
	We conduct extensive experiments on computer vision tasks following recent benchmarks \citep{ovadia2019can} \citep{hein2019relu}. We demonstrate that \our{} is competitive with the state-of-the-art on in-distribution calibration, robustness under dataset shifts, and detection of out-of-distribution data. An additional ablation study shows the impact of the aggregation scheme and model calibration on overall performance.
	
	Our contribution can be summarized as follows:
	\begin{itemize}
		\item We introduce \our{} -- a post hoc method for uncertainty modeling, which can be applied at test time to a OVA classifier. Our theoretical result shows that the confidence score of \our{} is close to 0 for samples that are sufficiently far from in-distribution data.  
		\item We construct a novel and powerful exponential calibration to align \our{} confidence with model accuracy. The introduced technique allows us to obtain more reliable confidence estimates on in-distribution data.
		\item Experimental results demonstrate that \our{} obtains performance comparable to state-of-the-art methods on three fundamental uncertainty tasks: in-distribution calibration, robustness under dataset shift, and uncertainty prediction on OOD data.
	\end{itemize}
	
	The paper is organized as follows. First, in Section \ref{sec:relwork} we describe the relevant literature. Then, in Section \ref{sec:model}, we discuss the main features of the proposed model. We introduce the \our{} confidence score, explain its application to out-of-distribution tasks, and describe the exponential calibration. In Section \ref{sec:experiments} we discuss \our{} performance in three experiments -- the in-distribution classification, robustness under dataset shift, and the identification of out-of-distribution data. The experiments contain a discussion of the results, a statistical analysis of the performance of the \our{} model, and an ablation study that focuses on identifying the most significant elements of the model. The work is concluded in Section \ref{sec:conclusion} with a discussion and future work recommendations. 
	
	\section{Related Work}
	\label{sec:relwork}
	The topic of uncertainty estimation is extensive, with a vast amount of published literature. A complete survey on uncertainty modeling in neural networks can be found in \cite{gawlikowski2021survey}. Below, we recall the papers that are the most closely related to our approach. In particular, we focus on the Bayesian approach as the most theoretically sound but not scalable, the role of calibration methods, alternative approaches to out-of-distribution detection, and the relevance of OVA models.
	
	\subsection{Bayesian methods} 
	
	One group of methods uses the Bayesian framework, which naturally expresses the uncertainty of the prediction \citep{blundell2015weight}. The main idea relies on estimating the probability distribution on the weights of the neural network, making the model robust to perturbations. Since exact Bayesian inference on the weights of a neural network is intractable, several approximations have been introduced \cite{yao2019quality}. Although theoretically attractive, these methods do not scale easily to larger models \cite{hernandez2015probabilistic}. Gal and Ghahramani \cite{gal2016dropout} propose using dropout as an approximation of Bayesian inference to obtain uncertainty estimates. In contrast to Bayesian neural networks, this approach achieves a similar effect at no additional cost. Maddox et al. \cite{maddox2019simple} improve this even further by approximating the posterior with averaged SGD (stochastic gradient descent) iterates. More precisely, by making use of the information contained in the SGD trajectory, this model approximates the posterior distribution over the weights of the neural network in the form of a Gaussian distribution. Although the above Bayesian approaches obtain good results, they require significant modifications to the training procedure and are computationally expensive compared to standard (non-Bayesian) neural networks.
	
	Lakshminarayanan et al. \cite{lakshminarayanan2016simple} proposed an alternative to Bayesian methods, which is based on creating an ensemble (committee) of neural networks. If the base models are sufficiently diverse, their aggregated predictions significantly improve the uncertainty estimates. The method can be additionally improved using adversarial training. An approximate Bayesian version of ensembling was introduced by Pearce et al. \citep{pearce2020uncertainty}. Although the ensemble methods still achieve state-of-the-art results in most benchmarks \citep{ovadia2019can} \citep{gustafsson2020evaluating}, it should be noted that training and evaluating many models is also computationally demanding.

	\subsection{Calibration methods}
	
	It is possible to use a standard model and then add a post-processing step to calibrate the outputs (which aligns predictive probabilities with model accuracy). Gao et al. \cite{guo2017calibration} show that modern neural networks are highly miscalibrated and explore multiple calibration methods. One class of methods considered there learns a logit transformation applied before they are fed to the softmax function. Temperature scaling, where a single scalar multiplicative factor is trained, stands out among multi-class calibration methods evaluated for its results, surpassing both vector scaling and matrix scaling. Kull et al. \cite{kull2019beyond} introduce a method to directly calibrate the softmax outputs instead. Zhang et al. \cite{zhang2020mix} propose to create an ensemble of calibration transformations of the same type and additionally suggest composing it together with a non-parametric calibration method. It is also possible to use a function approximated by a neural network as a parametric calibration method, provided that it preserves the internal order of the predictions \citep{rahimi2020intra}. Finally, Gaussian processes were also examined as calibration methods \citep{milios2018dirichlet} \citep{wenger2020non}.
	
	Another recent approach to uncertainty estimation relies on replacing the conventional cross-entropy function by focal loss \citep{mukhoti2020calibrating}. The idea behind focal loss is to direct the network's attention to samples for which it is currently predicting a low probability for the correct class. Seo et al. \citep{seo2019learning} enhance the typical cross-entropy loss between predictions and ground truth with the cross entropy between the predictions and the uniform distribution, forcing the network to construct as uniform distribution as possible. Kumar et al. \citep{kumar2018trainable} design a kernel-based function, which can be optimized during training and serves as a surrogate for the calibration error.

	\subsection{Out-of-distribution detection}
	
	Although calibration allows for aligning predictive probabilities with model accuracy on in-distribution data, this technique may not be adequate for detecting OOD data \citep{leibig2017leveraging}. It was shown that out-of-distribution examples can have higher maximum softmax probabilities than in-distribution samples \cite{hendrycks2016baseline}. Multiple methods attempt to correct that by using postprocessing to adjust softmax outputs \cite{devries2018learning, liang2017enhancing}. Alternative approaches integrate generative models, such as Generative Adversarial Networks (GANs) or Variational Autoencoders (VAEs), into the training procedure to discriminate in-distribution from OOD data \citep{lee2017training, wang2017safer}. It was later argued that generative models might return high confidence for inputs outside of the class they are supposed to model \citep{nalisnick2018deep, hendrycks2018deep}. To avoid uncertain decisions, a rejection option was introduced into the classifier \cite{tewari2007consistency, carlini2017adversarial}. In the case of images, the authors of \cite{hein2019relu} enforce low confidence on OOD inputs with a procedure similar to adversarial training. These authors additionally prove that ReLU networks with softmax output produce high-confidence predictions far away from the training data. A detailed survey of OOD detection methods can be found in \cite{yang2021generalized}.

	\subsection{One-vs-all models}
	
	One-vs-all classification models are promising in estimating uncertainty prediction because they naturally encode the "none of the above" class. In most cases, the confidence score of an OVA classifier is calculated as the maximum class probability \citep{zadrozny2002transforming, shu2017doc, padhy2020revisiting}. Our paper shows that inspecting the probability of a single class is not sufficient for quantifying uncertainty because a high probability of at least two classes leads to overall high confidence of OVA. To reduce this negative effect, the authors of \cite{franchi2020one} build an ensemble of OVA and softmax models. As in a typical ensemble approach, the confidence is high only if both models agree on the predictions. OVA classifier can also be combined with distance-based logits \cite{padhy2020revisiting} which 
	take the density of data into account and, in consequence, work better for detecting OOD samples \cite{macedo2019isotropy}. 
	
	In contrast to these works, we combine the predictive probabilities of all classes to define the confidence score. The proposed confidence score represents the probability that the model returns a given class and does not return other classes. Consequently, the confidence is high if and only if one class has a high probability and the others are close to zero. To deal with in-distribution samples, we additionally introduce exponential calibration, which is capable of approximating any monotonic function.
	
	\section{Model Description}
	\label{sec:model}
	
	In this section, we provide a detailed description of the \our{} method. First, we formulate the classification task and explain the issues that arise when the distribution of the dataset changes. Then, we briefly describe the building blocks of the \our{} approach. The next subsection demonstrates how the proposed method outputs desirable confidence in out-of-distribution datasets. The last subsection describes exponential calibration as the final postprocessing step that aims to correct both the overconfidence of the classifier and the underestimation arising from a multiplicative form of the \our~confidence score.
	
	\subsection{Problem statement} 
	
	We consider a multi-class classification problem, in which every example $x \in \R^D$ is associated with a class $y \in \{1,\ldots,K\}$. Given the input $x$, we use a neural network to model the predictive distribution $p_\theta(y|x)$ over the classes, where $\theta$ are the parameters of the neural network. Throughout the paper, we drop the index $\theta$ to simplify the notation.
	
	In practice, the neural network is trained on a subset $X_{in} \subset \R^D$, which represents the in-distribution (domain) data. In consequence, during training the model learns the in-domain distribution $p(y|x, x \in X_{in})$, which might not generalize well to the out-of-distribution (OOD) images $X_{out}$. Even if the model is evaluated on the in-distribution data, we frequently observe overconfidence, which means that the model is not well calibrated. In other words, the confidence of the model related to the predictive distribution significantly exceeds its accuracy. 
	
	We are interested in constructing a model whose predictive distribution is well calibrated to the model accuracy. This distribution has to be robust to dataset (covariate) shifts. Consequently, the predictive distribution should translate into a confidence score that indicates whether the model can make the correct decisions. Such a property allows one to deal with OOD data.
	
	\begin{figure}[h!]
		\centering
		\includegraphics[width=0.6\textwidth]{ 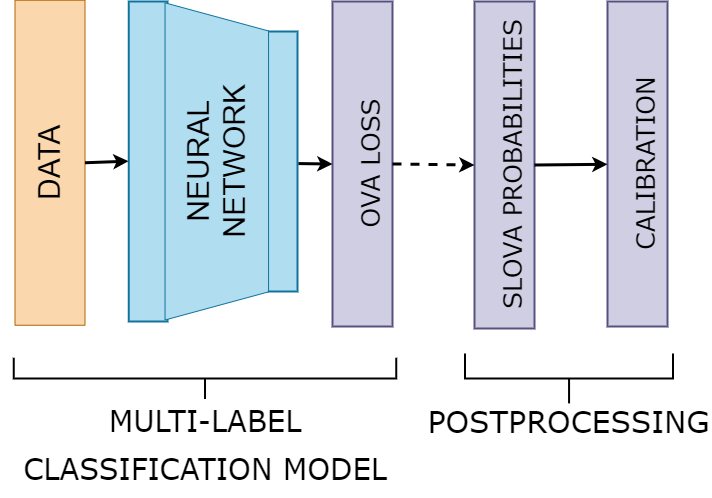}
		\caption{\our{} is a post hoc method consisting of a OVA classifier of choice (for example a pretrained neural network) and a post-processing step where the OVA predictive probabilities are transformed to \our{} calibrated probabilities. The \our{} probabilities reflect the confidence of in-distribution and out-of-distribution data.}
		\label{fig:intro2}
	\end{figure}
	
	\subsection{Model overview} 
	
	To achieve these goals, we propose the \our~ model shown schematically in Figure \ref{fig:intro2}. It is a two-stage post hoc method consisting of: a) a One-Vs-All (OVA) classifier and b) a postprocessing step that transforms OVA into \our{} probabilities.
	
	Firstly, we employ the OVA classification model, in which the predictive distribution of a given class is independent of other classes. OVA obtains comparable accuracy to softmax models but is much more flexible and can be used in multi-label classification problems, where the example can be labeled with more than one class. Since we are working in a single-label situation (only one class is correct), we redefine the predictive probabilities of OVA.
	Secondly, \our{} transforms the probabilities by compensating for values of other classes, reducing certainty when the model returns multiple classes or does not return any class (as depicted in Figure \ref{fig:intro}).
	In Theorem \ref{thm:main}, we prove that \our{} improves the confidence score on OOD data.
	
	While the introduced upgrade allows us to return low confidence for OOD data, we additionally need to calibrate the returned probabilities to match the model confidence with its accuracy. For this purpose, we introduce the exponential calibration function, whose form suits well the multiplicative form of \our{} predictive probabilities.
	
	The following parts describe the details of the proposed model.
	
	\subsection{Multi-label One-Vs-All classifier} 
	
	Let us recall that the predictive probabilities in OVA are described by $K$ sigmoid functions:
	\begin{equation*}
		p(k|x) = \frac{1}{1+\exp(-f_k(x))},
	\end{equation*}
	where $f(x)=[f_1(x),\ldots,f_K(x)]$ represents the network embedding of data point $x \in \R^D$ into $K$-dimensional logit space. The $k$-th sigmoid defines an individual Bernoulli distribution representing the probability that $x$ belongs to the $k$-th class. Since OVA allows for encoding more than one positive label for a given example, it is commonly used in multi-label problems. In the case of a typical single-label classification, where exactly one label is correct, OVA is trained by minimizing the standard cross-entropy function:
	\begin{equation}
		\begin{aligned} \label{eq:loss}
			L(x,y) & = -\log \left( p(y|x) \prod_{k \neq y} (1-p(k|x)) \right)\\
			& = -\log p(y|x) - \sum_{k \neq y} \log (1- p(k|x)),
		\end{aligned}
	\end{equation}
	for input $(x,y)$.

	The confidence score of the softmax model (or other single-label models) is usually defined as the maximal class probability \citep{guo2017calibration}. Following this idea directly, previous approaches to using OVA for uncertainty estimation \citep{padhy2020revisiting, shu2017doc} define OVA confidence by
	\begin{equation}\label{eq:confold}
		\confold(x) = \max_{k=1,\ldots,K} p(k|x).
	\end{equation}
	Observe, however, that the above formula does not take into account the fact that precisely one label is correct for every example. As mentioned, OVA suits well to multi-label situations and can return a high probability for more than one class. If more than one class has a high probability in single-label classification, then the model confidence should be low because the model is unsure which class to pick. Nevertheless, the formula \eqref{eq:confold} ignores such a situation and returns high confidence if at least one class is highly probable.
	
	\subsection{\our{} confidence} 
	
	Motivated by the above reasoning, we calculate the correct value of the probability that $x$ belongs to the $k$ th class in the single-label case. This is the probability that the model assigns $x$ to the $k$-th class and not to any other classes, which is given by:
	\begin{equation} \label{eq:pred}
		\Pour{}(k|x) = p(k|x) \prod_{j \neq k} (1-p(j|x)).
	\end{equation}
	This formula uses the information that we work in a single-label case and is consistent with the loss function \eqref{eq:loss}. Consequently, the \our{} confidence score is defined by:
	\begin{equation} \label{eq:conf}
		\confour(x) = \max_{k=1,\ldots,K} \Pour{}(k|x).
	\end{equation}
	This score is high if only one class has high probability while the others are negligible.

	\begin{remark} \label{rem:multi}
		Our idea of confidence for an OVA classifier can also be extended to the multi-label case. Recall that the OVA classifier returns a positive label at the $k$-th position if $p(k|x) \geq 1/2$. As a consequence, the confidence score in multi-label situation can be naturally defined by:
		$$
		\prod_{k: p(k|x) \geq 1/2} p(k|x) \cdot \prod_{j: p(j|x) < 1/2} (1-p(j|x)).
		$$
		In the extreme case, the OVA classifier may not return any label if $p(k|x) < 1/2$ for all $k=1,\ldots,K$. The probability that "none class is correct" is thus given by:
		\begin{equation*}
			\Pour{}(\mathrm{none}|x) = \prod_{k=1}^K (1-p(k|x)).
		\end{equation*}
		
		We leave the further discussion on quantifying the confidence in multi-label case for future work.
	\end{remark}
	
	In the next subsection, we continue with a theoretical analysis which shows that the \our{} method correctly outputs low confidence scores when evaluated on the out-of-distribution data. 
	
	\subsection{Theoretical analysis on OOD data}
	
	It is well known that softmax models tend to have overconfident predictions \cite{guo2017calibration}, especially for examples far away from the data distribution. In particular, it was shown that the confidence of neural networks with ReLU activation and softmax output converges to 1 for OOD data \cite[Theorem 3.1]{hein2019relu}. It partially follows from the fact that the softmax model always spreads the whole probability mass to all classes because there is no event like "none class is correct". 
	
	We show that such a limiting behavior does not hold for \our{} \eqref{eq:conf}, but is likely to occur for the typical OVA confidence score \eqref{eq:confold}.
	
	\begin{theorem} \label{thm:main}
		Let $f$ be a neural network with ReLU activation functions on hidden layers and sigmoid output implementing $K$-class OVA model. We assume that the distribution of classes is uniform. 
		
		Then for almost all $x \in \R^D$, we have $\confour(\alpha x) \to 1$ with probability $\frac{1}{2^K}$, as $\alpha \to \infty$. In the same case, $\confold(\alpha x) \to 1$ with probability $1-\frac{1}{2^K}$, as~$\alpha \to \infty$.
	\end{theorem}
	
	\begin{proof}
		Let us first recall that a feedforward neural network $f:\R^D \ni x \to (f_1(x),\ldots,f_K(x)) \in \R^K$ that uses piecewise affine activation functions (e.g. ReLU, leaky ReLU) and is linear in the output layer, can be rewritten as continuous piecewise affine functions. Thus there exists a finite set of polytopes $\{Q_l\}_{l=1}^R$ such that $\bigcup_{l=1}^R Q_l = \R^D$ and $f(x) = V^lx+a^l$ is the piecewise affine representation of the output of a ReLU network on $Q_l$.
		
		By \cite[Lemma 3.1]{hein2019relu}, there exists a region $Q_t$ with $t \in \{1,\ldots,R\}$ and $\beta$ such that for all $\alpha \geq \beta$ we have $\alpha x \in Q_t$. Let $f(x) = V^t x + a^t$ be the affine form of the ReLU classifier $f$ on $Q_t$. By making use of \cite[Theorem 3.1]{hein2019relu} and the fact that the sigmoid is a special case of the softmax function for two classes, we get that the sigmoid output $p(k|\alpha x)$ is arbitrarily close either to 1 or 0, for all $k=1,\ldots,K$ and sufficiently large $\alpha$. 
		
		Observe that $\confour(\alpha x) \to 1$ iff. there exists exactly one $k$ such that $p(k|\alpha x) \to 1$ and $p(i|\alpha x) \to 0$, for all $i \neq k$, as $\alpha \to \infty$. Assuming a uniform distribution of classes, this situation occurs with probability $\frac{1}{2^K}$. On the other hand, $\confold(\alpha x) \to 1$, as $\alpha \to \infty$, with probability $1-\frac{1}{2^K}$, because the probability that all $K$ sigmoids converge to 0 equals $\frac{1}{2^K}$.
	\end{proof}
	
	\begin{figure}[t]
		\centering
		\includegraphics[width=0.8\linewidth]{ 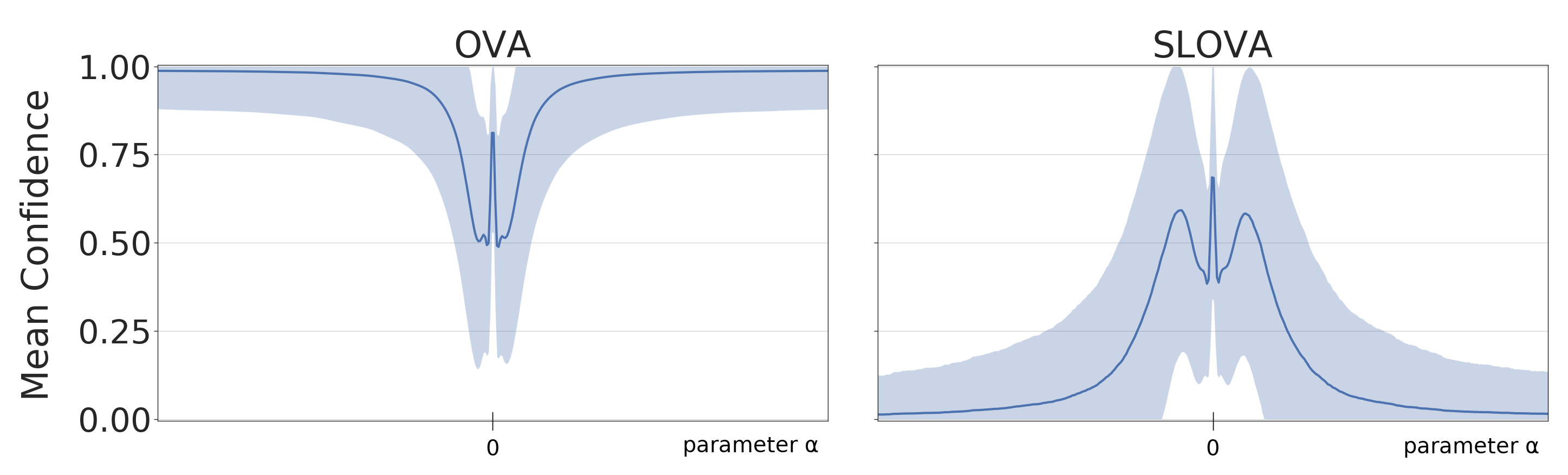}
		\caption{Confidence scores of OVA and \our{} on OOD data illustrating Theorem \ref{thm:main}. We show both the mean and the standard deviation of confidence scores for both models evaluated at an input $x + \alpha x_{\textrm{rand}}$ composed of a CIFAR-10 test point $x$ and a random vector $x_{\textrm{rand}}$. The $\alpha$ parameter plotted on the horizontal axis controls the strength of the distributional shift. In contrast to OVA, whose confidence saturates at high values of $\alpha$, the confidence of \our{} converges to 0 on OOD data as expected.}
		\label{fig:random_dir_mean}
	\end{figure}

	To illustrate the implications of the above result, let us consider a typical task with 10 classes such as the CIFAR-10 dataset. For OOD data $x \in X_{out}$, which are sufficiently distant from in-distribution samples, we have $\confour(x) \approx 1$ with probability $\frac{1}{2^{10}} < 0.001$. To construct such OOD samples it is sufficient to select almost any direction $z \in \R^D$ and put $x = \alpha z$, where $\alpha \in \R$ is a sufficiently large scalar. For analogical OOD samples, $\confold(x) \to 1$ with probability greater than 0.999. Consequently, a typical OVA classifier behaves analogously to a softmax classifier, while \our{} does not inherit this drawback.
	
	Figure \ref{fig:random_dir_mean} shows the mean confidence of OVA and \our{} for inputs localized on a fixed 1-dimensional affine subspace $x + \alpha x_{\textrm{rand}}$. The middle point $\alpha=0$ is a test sample $x$, which represents the in-distribution data. The plot is averaged over many random subspaces $x_{\textrm{rand}}$. It is evident that OVA, on average, returns high confidence for large $\alpha$, while \our{} converges to a small value. Observe the rapid drop of confidence around the test point for OVA and \our{}, which is caused by the decrease in the predictive probability of the dominant class (the probabilities of other classes are still small). Next, the predictive probabilities of some other classes are increasing, resulting in a slight increase of \our{} confidence. Finally, we observe a decrease of \our{} confidence because predictive probabilities of most classes converge to 0 while others converge to 1. OVA confidence converges to 1 since it relies only on the dominant class. Analogical comparison for the two-dimensional subspace is presented in Figure \ref{fig:random_plane_mean}.
	
	\subsection{Exponential calibration}
	\label{sec:expcal}
	We now describe calibration as the final postprocessing step of the \our~method. Using $\Pour{}(k|x)$ as predictive probabilities for $k=1,\dots,K$ reduces the overconfidence of a typical OVA classifier when more than one class has a high probability. However, $\sum_{k=1}^K \Pour{}(k|x) \leq 1$ because various multi-label events are additionally encoded; see Remark \ref{rem:multi}. As a consequence, there exists a discrepancy between the predicted probabilities (confidence) and the true probabilities (accuracy). The middle point in Figure \ref{fig:random_dir_mean} illustrates the underestimation of \our{} predictive probabilities on in-distribution data. Rather than normalizing the obtained predictive probabilities, which is only a partial solution, we instead use calibration techniques that align the expected confidence with accuracy.

	To formally define the calibration, let $\hat{p}(k|x)$ be the probability that the model returns the class $k$ for the input example $x$ and let $\mathbb{P}$ be the \emph{true} probability of classes. We say that the model is calibrated if the true probability $\mathbb{P}$ of class $k$ given the output probability $\hat{p}$ coincides with the output probability for that class:
	\begin{align}
		\label{eq:calform}
		\mathbb{P}\left (k | \hat{p}(x) \right ) = \hat{p}(k|x) \text{, for } k \in \{1,\ldots,K\}.
	\end{align}
	
	A typical way to calibrate the classification model relies on transforming the output probabilities $\hat{p}$ into calibrated ones $p$. Based on the multiplicative form of $\Pour{}$, we use the exponential calibration, which, we argue in \ref{appendix:expcal}, is capable of "inverting" the multiplications. Formally, we use the following parametric transformation:
	\begin{align}
		\label{eq:calibr}
		p_k = c(\hat{p}_k) = \sum_{i=1}^M \beta_i {\hat{p}_k}^{\alpha_i},
	\end{align}
	where $\alpha_i >0, \beta_i \in (0,1)$ are trainable parameters such that $\sum_{i=1}^M \beta_i = 1$. Due to the choice of $\alpha_i, \beta_i$ in the above ranges, the introduced transformation is monotonic and is a map on the unit interval. In the \ref{appendix:expcal} we argue that the exponential map is a natural choice to invert the multiplicative character of \our{} probabilities. Moreover, it was shown that the above exponential transformation is capable of approximating any completely monotonic function \citep{kammler1979least, mcglinn1978uniform}. 
	
	To calibrate \our{}, we form a calibration training set $X_{cal}$ from the validation dataset $X_{val}$. For each example $(x,y) \in X_{val}$, we gather all $K$ pairs of uncalibrated output probabilities and their corresponding one-hot encoded binary labels $\left ( \Pour{}(k|x), \delta_{k,y} \right ) \in X_{cal}$, which results in $|X_{cal}| = K|X_{val}|$ pairs in total.
	
	Next, we sort the pairs by the output probabilities $\Pour{}$:
	\begin{align*}
		\left \{ 
		\begin{array}{ll}
			P \vcentcolon= & \textrm{sort}(\Pour{}(k|x)),\\
			D \vcentcolon= & \textrm{sort}_{\Pour{}(k|x)}(\delta_{k,y}),
		\end{array}
		\right . ,\quad 
		k \in (1...K), \quad  (x,y) \in X_{val};
	\end{align*}
	and apply sliding windows of size $n$ to compute moving averages of the sorted data:
	\begin{equation*}
		\overline{\pi}_i = \frac{1}{n} \sum_{j=i}^{i + n-1} P_j\mathrm{, } \quad \overline{b}_i = \frac{1}{n} \sum_{j=i}^{i + n-1} D_j.
	\end{equation*}
	This results in $N = |X_{cal}| - n + 1$ averaged points, and we fix the window size $n=|X_{cal}|/100$ so that there are enough points within each window. Intuitively, $\overline{b}_i$ is the fraction of positive labels, while $\overline{\pi}_i$ represents the average confidence of the model in that window. Lastly, we select an equally spaced subset of the averaged pairs $(\overline{\pi},\overline{b})$ so that the total number of pairs in the dataset is $n_b$. The parameters $\alpha,\beta$ from \eqref{eq:calibr} are directly optimized by minimizing the mean square error:
	\begin{equation*}
		\sum_{i=1}^{n_b} (c(\overline{\pi}_i) - \overline{b}_i)^2.
	\end{equation*}
	
	Note that calibration methods are normally trained and evaluated on in-distribution data. This means that calibration usually does not improve and can actually increase confidence on out-of-distribution data. To prevent this, we additionally augment $X_{val}$ with random noise samples that are labeled "none", that is, $\delta_k = 0$ for $k \in (1..K)$ (only zeros in one-hot representation). We use $[0, 1]$ uniform distribution for the image data as we find that it works adequately and add $\lfloor0.1|X_{val}|\rfloor$ of such samples.
	
	The choice of hyperparameter $M$ affects the capacity and computational cost of the transformation. Note, however, that the cost is negligible compared to the cost of the forward pass of the base model. In \ref{appendix:calstability} we provide stability studies of exponential calibration as we vary the number of adjustable parameters $M$ and the size of the calibration dataset $n_b$.
	
	\section{Experiments}
	\label{sec:experiments}
	To evaluate both the confidence calibration and the out-of-distribution aspects of uncertainty estimation, we focus on three tasks presented in the following sections. First, we measure the model calibration on in-distribution data. Then we repeat the same experiment under distribution shift. We then proceed to examine the behavior of our model on out-of-distribution data. Lastly, we analyze the significance of the calibration step in the final section.
	
	\subsection{In-distribution and shifted data}
	\label{sec:expshift}
	\paragraph{Setup}
	To examine the model calibration on in-distribution and distorted data, we follow a recent benchmark \cite{ovadia2019can} and compare \our{} with the following baselines:
	\begin{itemize}
		\item \textit{Softmax}: Maximum softmax probability \citep{hendrycks2018benchmarking},
		\item \textit{Temp Scaling}: Softmax with calibration defined by the temperature scaling \citep{guo2017calibration},
		\item \textit{Dropout}: Monte Carlo Dropout \citep{gal2016dropout},
		\item \textit{Ensembles}: Ensemble of $M$ neural networks trained independently on the entire dataset \citep{lakshminarayanan2016simple},
		\item \textit{SVI}: Stochastic Variational Bayesian Inference for deep learning \citep{blundell2015weight, graves2011practical},
		\item \textit{LL SVI}: Mean field stochastic variational inference on the last layer only \citep{riquelme2018deep},
		\item \textit{LL Dropout}: Dropout applied only to activations before the last layer \citep{riquelme2018deep},
		\item \textit{OVA DM}: OVA classifier with distance-based logits \citep{padhy2020revisiting}.
	\end{itemize}
	
	In this setting, we train all models on the CIFAR-10 train set and subsequently calibrate them on the validation set. In-distribution calibration is evaluated on the original test set. To investigate the robustness under dataset shift, we follow Hendrycks \& Dietterich \cite{hendrycks2018benchmarking} and use CIFAR-10-C produced by applying 16 types of distortions to CIFAR-10 images, each taken with 5 degrees of intensity.
	
	We use the ResNet20 architecture \citep{he2016deep} that achieves state-of-the-art results on typical computer vision tasks, but is otherwise an arbitrary choice to compare methods. We train it with the Adam optimizer \citep{kingma2014adam} for $250$ epochs with batch size $512$. The train dataset is augmented by random crops and rotations, the learning rate is piecewise constant with two $0.1$ drops in epochs $125$ and $185$. Before training, we perform a random hyperparameter search of $20$ trials with learning rates sampled uniformly from an interval $[ 10^{-3},1]$, momentum $\beta_1 \in [0.85,0.99]$, and stability parameter $\epsilon \in [10^{-8},10^{-5} ]$. Calibration was performed for $M=20$ parameters trained for $20$ epochs with the size of the calibration dataset $n_b=4000$.
	\begin{figure*}[h!]
		\centering
		\includegraphics[width=\linewidth]{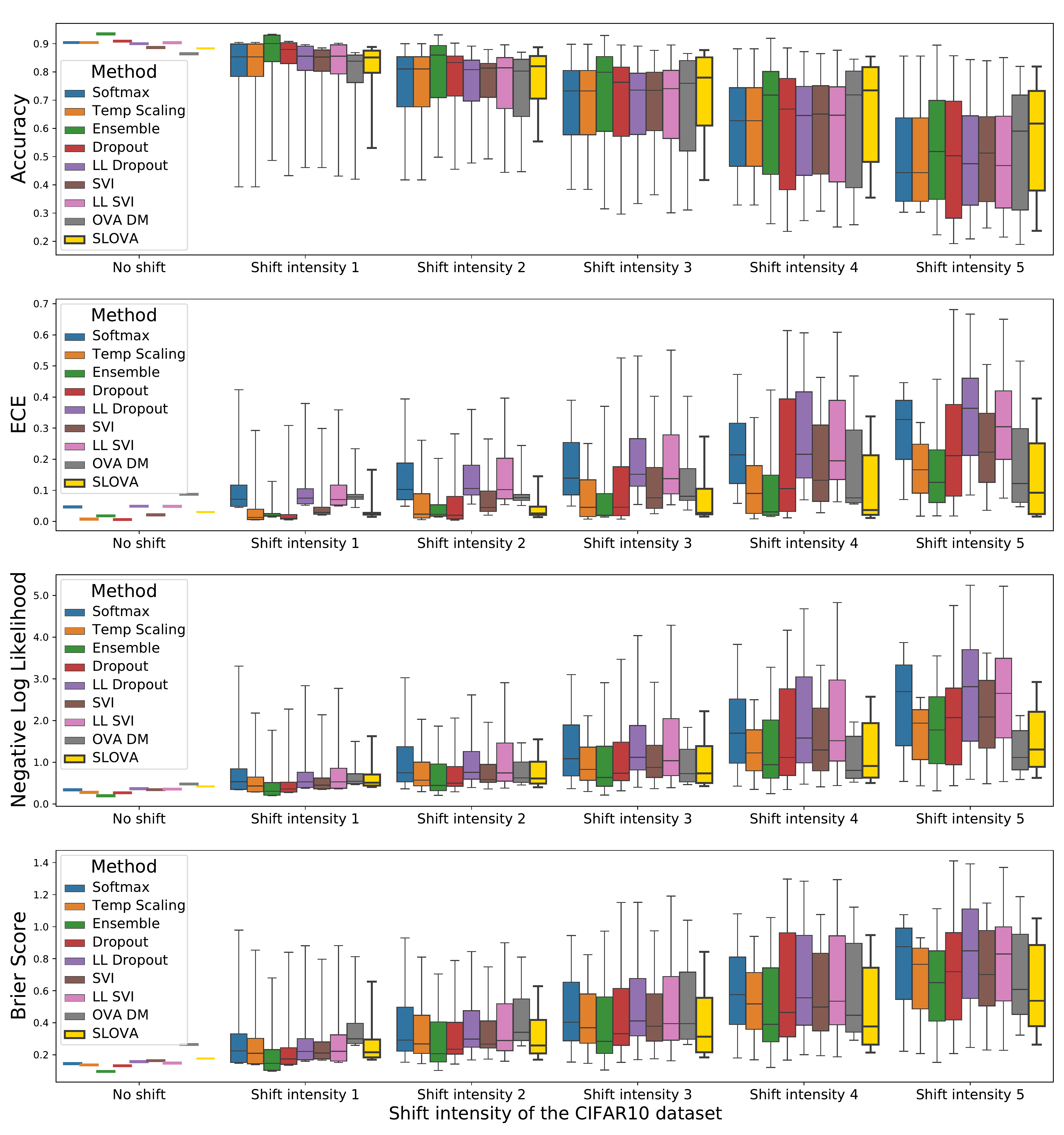}
		\caption{Accuracy and three uncertainty-aware metrics: ECE, NLL, and Brier score for the CIFAR-10 and its distorted version. On the horizontal axis we group model outputs with 5 degrees of shift intensity constructed following Hendrycks \& Dietterich \cite{hendrycks2018benchmarking}. Despite its simplicity, \our{} retains high accuracy and low Brier score across all shift intensities while ECE and NLL are comparable to more complex state-of-the-art models. Based on these results, in Figure \ref{fig:stattest-shift} we conduct a statistical performance test which ranks \our{} as the second-best giving up the place only to the computationally demanding ensembles. The proposed model is especially robust for higher values of shift severity where other models, in general, are prone to larger drops of performance.}
		\label{fig1:cifar10}
	\end{figure*}
	
	\paragraph{Metrics}
	We investigate the classification accuracy as well as three calibration-related metrics. The Expected Calibration Error (ECE) \citep{naeini2015obtaining} measures the discrepancy between predicted probabilities and empirical accuracy by grouping model predictions into $M$ interval bins $B_1,\ldots,B_M$ and calculating the expected difference between the accuracy $\acc(B_i)$ and confidence $\conf(B_i)$ over the bins:
	\begin{equation*}
		ECE = \sum_{i=1}^M \frac{|B_i|}{M} |\acc(B_i) - \conf(B_i)|.
	\end{equation*}
	We also consider the Brier score \citep{brier1950verification} -- the squared error of the predicted probability vector and the one-hot encoded true response; and
	the negative log-likelihood (NLL), which is also a proper scoring rule \citep{gneiting2007strictly}, but a less objective one, because it can overemphasize tail probabilities \citep{quinonero2005evaluating}.
	
	\paragraph{Results}
	The results presented in Figure \ref{fig1:cifar10} demonstrate that \our{} gives comparable accuracy and calibration scores to vanilla softmax on in-distribution data, but is more robust to dataset shifts. At the same time, it presents a notable improvement over softmax and its calibrated variant (Temperature Scaling) on highly distorted data in terms of all metrics. In the case of high intensity dataset shifts (level 4 and 5), the calibration scores (ECE, NLL, Brier) of \our{} are frequently better than the ones obtained by ensembles, which is currently considered the state-of-the-art method in uncertainty modeling. Note that complex methods, such as Ensemble, are significantly more computationally demanding than \our{}. While \our{} relies on using a single OVA model, Ensemble requires training and evaluation of multiple classifiers, which are also trained using adversarial examples. 
	
	The advantage of \our{} on highly distorted data can be explained by the fact that \our{} is designed not only to improve predictive uncertainty but also to detect OOD data. Strong shifts pushed the data far away from the true data distribution, which is similar to creating OOD samples. This experiment confirms our thesis that most methods are created to solve only one particular problem -- in-distribution calibration, robustness under dataset shifts, or OOD detection. In contrast, \our{} is the only method that gives satisfactory results on all of these tasks. 
	
	It is also worth noting that \our{} compares favorably with OVA DM, which is an upgraded version of the OVA classifier. OVA DM performs poorly for the in-distribution data (drops in accuracy observed both in our experiments and in the original article \cite{padhy2020revisiting}), a weakness that our method does not share. Nevertheless, in the case of high intensity dataset shifts, the performance of OVA DM is still lower than \our{}, which works well in all cases. 
	
	\paragraph{Statistical analysis}
	
	\begin{figure}[h!]
		\centering
		\includegraphics[width=1\textwidth]{ 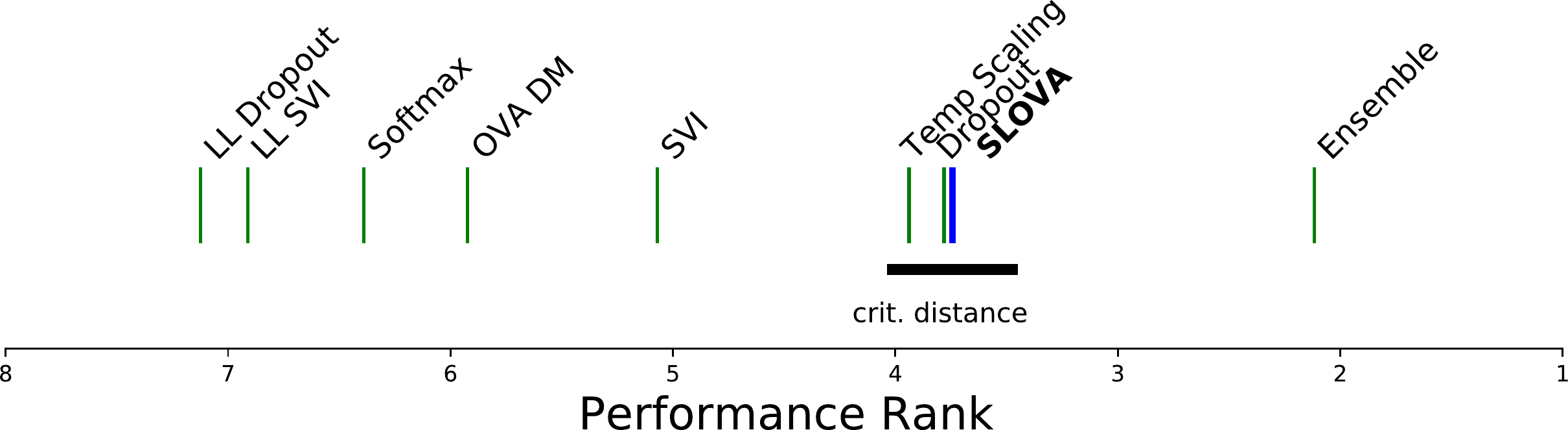}
		\caption{Visualization of statistical comparison performed on in-distribution calibration and dataset shifts. Methods which lie inside the critical distance around \our{} are statistically indistinguishable.}
		\label{fig:stattest-shift}
	\end{figure}
	
	To support our analysis, we perform statistical tests to verify the statistical difference between algorithms. Specifically, we combine the nonparametric statistical Friedman test with the Bonferroni-Dunn post hoc analysis \cite{demvsar2006statistical}. To this end, we aggregated the results of all the methods from Figure \ref{fig1:cifar10}, giving 81 in total. Next, we ranked the methods according to their performance. Given this ranking, the analysis consists of two steps:
	\begin{itemize}
		\item The null hypothesis is made that all methods perform equally and the observed differences are merely random. The hypothesis is tested by the Friedman test, which follows a $\chi^2$ distribution.
		\item Having rejected the null hypothesis, the differences in ranks are analyzed by the Bonferroni-Dunn test. The Bonferroni-Dunn tests are applied only between \our{} and other methods. 
	\end{itemize}
	
	Figure \ref{fig:stattest-shift} visualizes the results for a significance level of $\alpha=0.05$. The x-axis shows the mean rank for each method. The thick horizontal line illustrates the critical distance. Methods for which the difference in mean rank is lower than the critical distance are statistically indistinguishable.
	
	It is evident that the mean rank of \our{} is comparable to \emph{Dropout} and \emph{Temp Scaling}. Since the distance between their mean ranks is lower than the critical distance, the difference between these methods is statistically insignificant. At the same time, \our{} is statistically better than other models, except \emph{Ensemble}, which performs best across all cases.

	\subsection{Out-of-distribution data}
	
	\paragraph{Setup}
	We train a classification model on one dataset for the OOD experiment and then evaluate it on other out-of-distribution datasets. We reuse the setup from \cite{hein2019relu} and compare \our{} with the following methods:
	\begin{itemize}
		\item \textit{Softmax}: Maximum softmax probability,
		\item \textit{CEDA}: Softmax network trained with an additional loss term that penalizes non-uniform answers on additional OOD train data \citep{hendrycks2018deep, hein2019relu},
		\item \textit{ACET}: Softmax network also trained with the same loss term, but with OOD examples being generated adversarially \citep{hein2019relu},
		\item \textit{ReAct}: Softmax network with penultimate activations being clipped to a threshold in evaluation \citep{sun2021react}.
	\end{itemize}
	
	A ResNet model is trained on each of the three commonly used datasets: CIFAR-10, CIFAR-100, and SVHN. Random crops are used for all datasets for training, while random mirroring is additionally used for CIFAR-10 and CIFAR-100. We evaluate each model on the two remaining datasets along with the LSUN (classroom subset) and the ImageNet datasets. All inputs are $32$x$32$ pixels except for the ImageNet dataset, where the images are scaled to $224$x$224$ pixels. For ReAct, we follow the original authors and set the threshold $c$ so that it corresponds to the value of the $90$-th percentile of all activations from the penultimate layer \citep{sun2021react}.
	
	For \our{} we slightly tune the original hyperparameters of the softmax model to avoid a large drop in the accuracy of the OVA model. This is easily achieved by altering the learning rate and the weight decay hyperparameter $\lambda$ - we use the learning rate $0.02$, $0.03$, $0.02$, and $\lambda$ $0.01$, $0.02$, $0.01$ for CIFAR-10, CIFAR-100, and SVHN, respectively. The remaining hyperparameters are exactly the same as in the setup from \citep{hein2019relu}.
	
	\begin{table}[h!]
		\centering
		\footnotesize
		\caption{Mean Maximum Confidence for the out-of-distribution (OOD) experiment. Low confidence on OOD data is desired and \our{} produces top results on 9 out of 12 cases.}
		\begin{tabular}{cccccc}
			\toprule
			\textbf{Method} & \textbf{Error} & \multicolumn{4}{c}{\textbf{Mean Maximum Confidence} }\\
			\midrule
			\multicolumn{2}{c|}{\textbf{SVHN}} & \textbf{CIFAR-10} & \textbf{CIFAR-100} & \textbf{LSUN} & \textbf{ImageNet}\\ \cline{3-6}
			Softmax & $3.53\%$ & $0.732$ & $0.730$ & $0.722$ & $0.568$ \\
			CEDA & $3.50\%$ & $0.551$ & $0.527$ & $0.364$ & $0.736$\\
			ACET & $3.52\%$ & $0.435$ & $0.414$ & $\textbf{0.148}$ & $0.421$ \\
			ReAct & $3.77\%$ & $0.671$ & $0.676$ & $0.673$ & $0.489$ \\
			\our{} & $3.29\%$ & $\textbf{0.376}$ & $\textbf{0.387}$ & $0.338$ & $\textbf{0.146}$ \\
			\midrule
			\multicolumn{2}{c|}{\textbf{CIFAR-10}} & \textbf{SVHN} & \textbf{CIFAR-100} & \textbf{LSUN} & \textbf{ImageNet}\\ \cline{3-6}
			Softmax  & $8.87\%$ & $0.800$ & $0.764$ & $0.738$ & $0.864$ \\
			CEDA & $8.87\%$ & $0.327$ & $0.761$ & $0.735$ & $0.660$ \\
			ACET & $8.44\%$ & $\textbf{0.263}$ & $0.764$ & $0.745$ & $0.242$\\
			ReAct & $9.19\%$ & $0.725$ & $0.733$ & $0.689$ & $0.713$ \\
			\our{} & $7.68\%$ & $0.644$ & $\textbf{0.640}$ & $\textbf{0.566}$ & $\textbf{0.180}$ \\
			\midrule
			\multicolumn{2}{c|}{\textbf{CIFAR-100}} & \textbf{SVHN} & \textbf{CIFAR-10} & \textbf{LSUN} & \textbf{ImageNet}\\ \cline{3-6}
			Softmax & $31.97\%$ & $0.570$ & $0.560$ & $0.592$ & $0.878$ \\
			CEDA & $32.74\%$ & $0.290$ & $0.547$ & $0.581$ & $0.542$ \\
			ACET & $32.24\%$ & $\textbf{0.234}$ & $0.530$ & $0.554$ & $0.601$ \\
			ReAct & $37.64\%$ & $0.426$ & $0.448$ & $0.510$ & $0.667$ \\
			\our{} & $32.01\%$ & $0.407$ & $\textbf{0.402}$ & $\textbf{0.386}$ & $\textbf{0.098}$ \\
			\bottomrule
		\end{tabular}
		\label{tab:ood_datasets_results}
	\end{table}

	\paragraph{Metrics}
	Similarly as in \cite{hein2019relu}, we report the test error and the mean maximal confidence (MMC) for each model. MMC is defined as the mean of maximal model predictions taken over all data points. In the case of \our{}, MMC is evaluated using $\max_{k=1,\ldots,K} \Pour{}(k|x)$,
	averaged over all $x$. Low MMC means low confidence, which is expected on OOD data.
	
	\paragraph{Results} 
	The results reported in Table \ref{tab:ood_datasets_results} demonstrate that \our{} gives significantly lower MMC scores than the typical softmax classifier on OOD detection task. It is interesting that ReAct, a recent OOD detector, performs only slightly better than softmax. At the same time, it is inferior to \our{} in all cases and has a significant reduction in classification accuracy for the model trained on CIFAR-100.
	
	Notably, ACET gives impressive scores in a few variants (SVHN/LSUN, CIFAR-10/SVHN, CIFAR-100/SVHN), but its results are very unstable. In particular, ACET trained on CIFAR-10 is able to almost perfectly detect SVHN and ImageNet samples, but it is unaware of examples from the CIFAR-100 and LSUN datasets. This behavior can be explained by the fact that ACET is trained using adversarial examples, which is a complex and quite difficult procedure to apply in practice. In this case, it performs even worse than CEDA, a similar method that uses noise instead of adversarial examples. 
	
	In terms of OOD detection, \our{} outperforms other methods in 9 out of 12 cases. Furthermore, this experiment shows that OVA classifiers provide accuracy comparable to that of softmax models in typical classification tasks. This is an important conclusion because most neural networks are based on softmax. Since \our{} performs similarly to softmax but is better in OOD detection and predictive uncertainty, it is safe to replace the softmax model with a \our{} classifier.
	
	\paragraph{Statistical analysis}
	
	\begin{figure}[h!]
		\centering
		\includegraphics[width=0.8\textwidth]{ 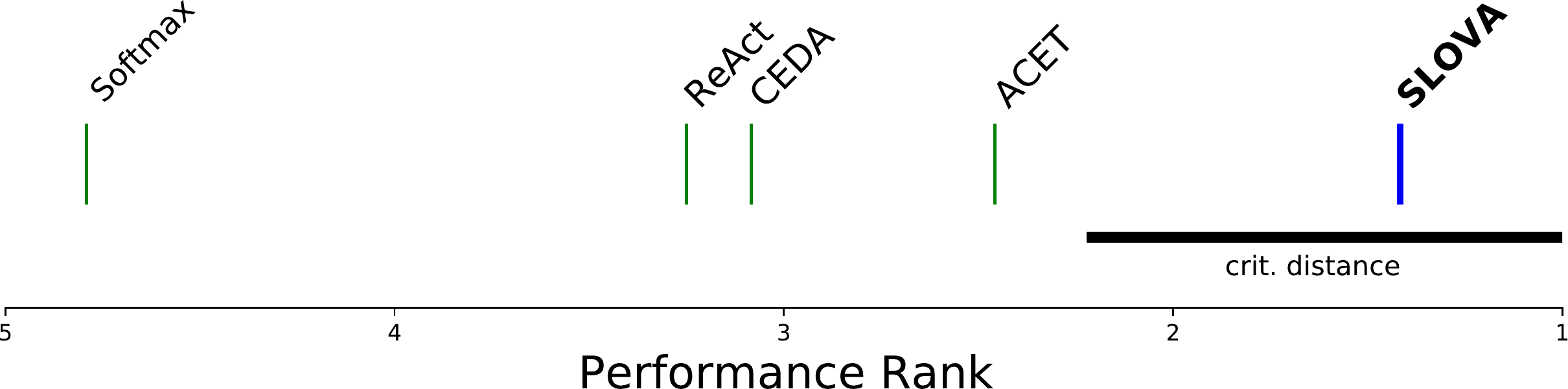}
		\caption{Visualization of statistical comparison performed on OOD detection. The difference between \our{} and other methods are statistically significant, which is indicated by the critical distance.}
		\label{fig:stattest}
	\end{figure}
	
	We repeat the statistical analysis for the OOD detection. Namely, we conduct the non-parametric statistical Friedman test with the Bonferroni-Dunn post hoc analysis for 12 combinations of datasets reported in Table \ref{tab:ood_datasets_results}.
	
	As can be seen in Figure \ref{fig:stattest}, \our{} received a rank close to 1, which confirms that it was superior in most cases. More importantly, the difference between \our{} and other models is statistically significant.

	\subsection{Ablation study} 
	
	We perform the following ablation study to analyze the impact of the proposed confidence score and the calibration separately. We group all the cases of in-distribution data and distributional shifts together (including shift intensities). First, we examine the standard OVA classifier, then \our{} without exponential calibration, and finally the complete \our{} model. For comparison, we also report the vanilla softmax model.
	
	\begin{figure}[h!]
		\centering
		\includegraphics[width=0.8\linewidth]{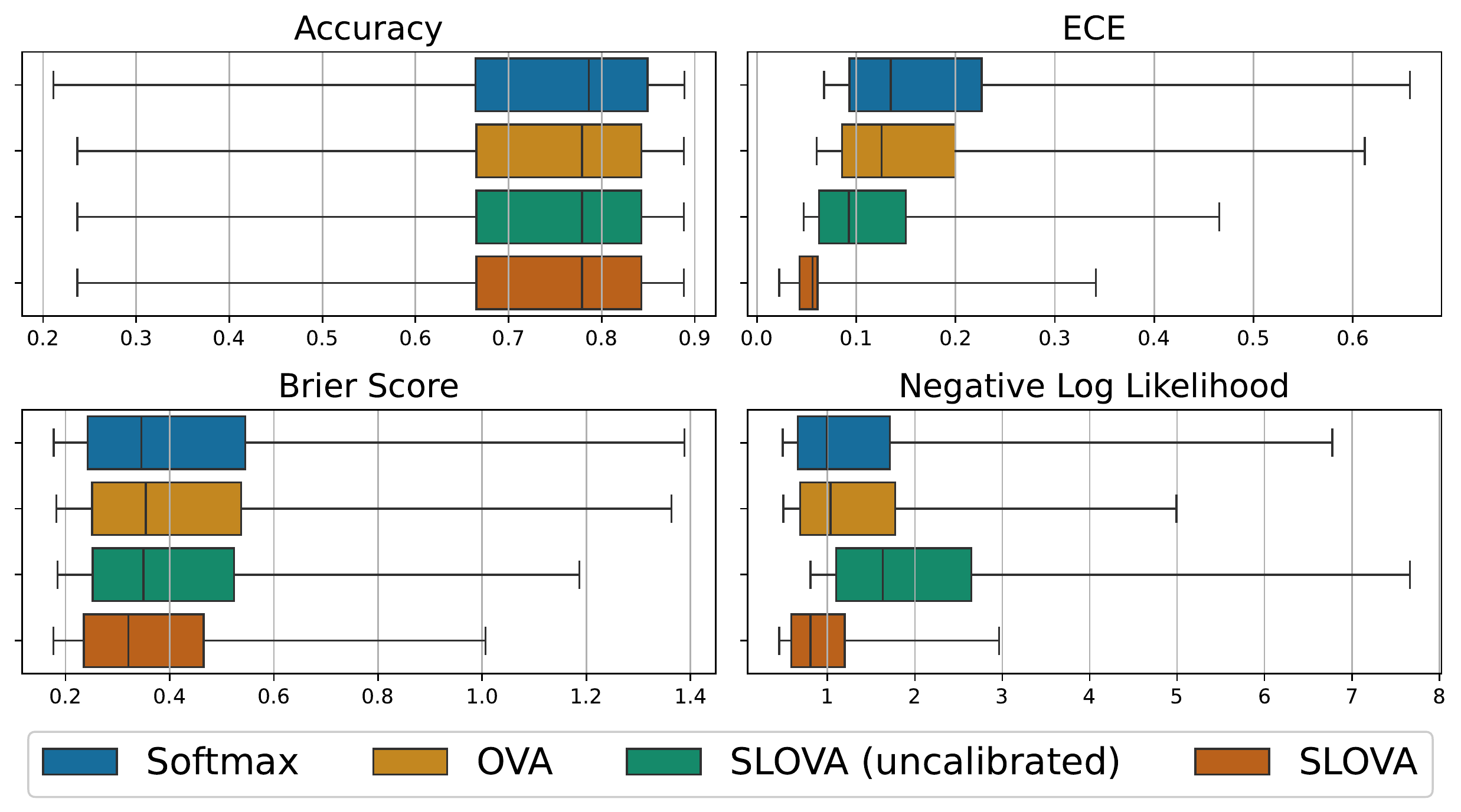}
		\caption{Ablation study on CIFAR-10-C testsets. It is evident that each step of our method successively improves the uncertainty estimation while preserving the model accuracy.}
		\label{fig3:ablation}
	\end{figure}
	
	The results reported in Figure \ref{fig3:ablation} confirm that subsequent modifications do not affect accuracy. This is expected behavior because both \our{} predictive probabilities and exponential calibration do not change the order of probability scores. As a consequence, the accuracy of the OVA model is preserved.
	
	It is evident that subsequent modifications of the OVA probabilities produced by \our{} successively improve the ECE score. The impact on the Brier score is smaller, but there is a slight positive trend. In the case of NLL, there is no evident trend, but a complete \our{} model obtains the best result with the lowest variance.
	
	The most significant difference observed in the case of the ECE metric can result from the fact that the exponential calibration was optimized using a type of ECE objective. Although all calibration metrics are related, the selected optimization objective is crucial for the final performance. In most studies, the ECE is considered the most reliable calibration metric, which motivated our selection. It is worth emphasizing that many methods (including temperature scaling) use log-likelihood to select calibration parameters because its optimization is significantly easier.

	\section{Conclusion}
	\label{sec:conclusion}
	\subsection{Summary}
	
	In this work, we presented \our{}, a simple and effective method of transforming OVA probabilities into reliable uncertainty predictions. Since OVA can be used in multi-label classification tasks, it has a natural ability to detect OOD samples by labeling them as "none of the above classes". As a consequence, \our{} goes beyond the properties of a typical softmax predictor and is better suited for modeling the distribution of OOD samples. We confirmed this claim theoretically. To quantify the uncertainty of the extended OVA classifier on in-distribution and shifted data, we introduced the exponential calibration, which allows us to align predictive probabilities with the accuracy of the model. To support our theses, we experimentally demonstrated that \our{} compares favorably against state-of-the-art methods on common benchmark tasks.
	
	\subsection{Discussion and limitations}
	
	The problem of uncertainty estimation can be divided into at least three separate tasks: in-distribution calibration, robustness under dataset shifts, and detection of OOD samples. Although all of these tasks are closely related, most methods focus on solving only one or two of them. We showed that combining exponential calibration with appropriate treatment of the OVA classifier gives a model that presents satisfactory performance on all of these problems. In particular, \our{} is one of the three best-performing methods on in-distribution calibration with optional dataset shifts. At the same time, it outperforms state-of-the-art baselines on OOD detection in most cases. In consequence, \our{} is a method that performs well in diverse applications, while other methods are directly designed only for a particular task.
	
	It is not obvious whether \our{} can be applied to the case of extreme classification. \our{} confidence score is defined as a product of OVA predictive probabilities. When the number of classes is high, but only one is correct, a single mistake can result in inaccurate or incorrect predictions. Such a disadvantage disappears in the multi-label classification, where more than one label is correct and we select the most probable combination of labels. A possible remedy in a single-label situation would be to project the class vector to a lower-dimensional space as it is commonly done in extreme classification \cite{yeh2017learning}.
	
	One could argue that a classification neural network implemented with the use of a softmax function is the standard approach that gives state-of-the-art performance. As a consequence, most of the available pre-trained models are softmax models. This is a considerable limitation, as even though \our{} is a post hoc method, it requires an OVA model. However, OVA classifiers usually give performance comparable to softmax, and both models can be trained effectively. In consequence, \our{} can be an alternative to softmax in applications where reliable uncertainty modeling is a desirable property.
	
	Finally, observe that Theorem \ref{thm:main} states that for ReLU activations, there is a $\frac{1}{2^K}$ probability of selecting a direction for which the network is overconfident with arbitrarily high $\alpha$. In practice, OOD samples can be distributed closer to in-distribution samples than $\alpha x$ with saturated confidence. Despite the limitations of these theoretical results, our experiments confirm that \our{} performs particularly well across several OOD benchmark tasks, strongly suggesting that the proposed method is useful for real-life use cases.

	\subsection{Future works}
	
	We showed that the predictive probability returned by the OVA classifier can be used to define confidence score in a single-label classification. As discussed in Remark \ref{rem:multi}, we can formulate an analogical confidence measure in a multi-label situation. In future work, we plan to explore this idea and verify it experimentally.
	
	Uncertainty estimation plays a crucial role in many areas of machine learning, including active learning, semi-supervised learning, or conditional computation. It would be interesting to verify whether the application of \our{} in these areas improves performance.
	
	An interesting direction for future work would be to improve the method by combining it with approaches that are orthogonal, and thus could be applied simultaneously. For example, it is well known that constructing an ensemble of networks improves the uncertainty scores \cite{lakshminarayanan2016simple}, yet the question of whether this fact is also true for \our{} remains unanswered. Another simple way to improve OOD detection performance is to expose the model to some OOD samples in the training process \cite{hendrycks2018deep}, and one could augment \our{} with this approach.

	\section*{Acknowledgements}
	The work of Bartosz Wójcik and Marek \'Smieja was supported by the National Science Centre (Poland) grant no. 2018/31/B/ST6/00993. Jacek Tabor carried out this work within the research project "Bio-inspired artificial neural networks" (grant no. POIR.04.04.00-00-14DE/18-00) within the Team-Net program of the Foundation for Polish Science co-financed by the European Union under the European Regional Development Fund. 
	
	\bibliography{main}

	\appendix
	
	\section{Justification of the exponential calibration}
	\label{appendix:expcal}
	We justify a particular form of the exponential calibration \eqref{eq:calform} whose aim is to "invert" the multiplication of many probabilities or to counteract the contractive character of this multiplication. We start off from an (uncalibrated) output of the \our{} model given in \eqref{eq:pred} for a generic input $x$:
	\begin{align*}
		\hat{P} = \hat{p}_1 (1-\hat{p}_2)... (1-\hat{p}_K),
	\end{align*}
	where we introduce a succinct notation for the uncalibrated probabilities $\hat{p}(i|x) \to \hat{p}_i, \hat{P}_{\textrm{\our}}(1|x) \to \hat{P}$. Without loss of generality, we pick out the first probability as the "correct" one. In general, since all $\hat{p}_i$'s are numbers between zero and one, the multiplicative \our{} probability $\hat{P}$ becomes, as $K$ increases, more likely to reach zero. For single-label examples, $\hat{p}_1 \approx 1$ and $\hat{p}_{i\neq 1} \approx 0$, which results in a (relatively small) drift away from $1$. However, it quickly becomes unacceptably large when the number of classes $K$ grows.
	
	We consider this multiplication-induced drift in a tractable, yet artificial case of complete ignorance where each partial probability $\hat{p}_i \sim U(0,1)$ is drawn from an independent uniform distribution. A simple calculation using the Mellin transform results in the distribution over \our{} probabilities:
	\begin{align}
		\label{eq:uncalrand}
		\mathbb{P}_{uncal}(\hat{P}) = \frac{1}{(K-1)!} 
		\left ( \log \frac{1}{\hat{P}}\right )^{K-1}
	\end{align}
	which becomes highly concentrated around $\hat{P}\sim 0$ even for a moderate number of classes $K=10$.
	
	On the other hand, calibration \eqref{eq:calibr} is a functional relation $P = c(\hat{P})$ that maps between the uncalibrated $\hat{P}$ and calibrated probabilities $P$. In the special case \eqref{eq:uncalrand}, a valid calibration $c$ is conceptualized as a function that reverts back to a form of complete ignorance, i.e., the initial uniform probability density:
	\begin{align*}
		\mathbb{P}_{cal}(P) = \int d\hat{P} \delta (P - c(\hat{P})) \mathbb{P}_{uncal}(\hat{P}) = 1, \quad \text{for}~ P \in (0,1).
	\end{align*}
	With the usual properties of the Dirac delta function, a transform achieving this is simply the cumulative distribution function:
	\begin{align}
		c(\hat{P}) = \int_{-\infty} ^{\hat{P}} \mathbb{P}_{uncal}(\hat{P}') d\hat{P}'.
		\label{eq:exactcal}
	\end{align}
	
	\begin{figure}[h!]
		\centering
		\includegraphics[width=.6\textwidth]{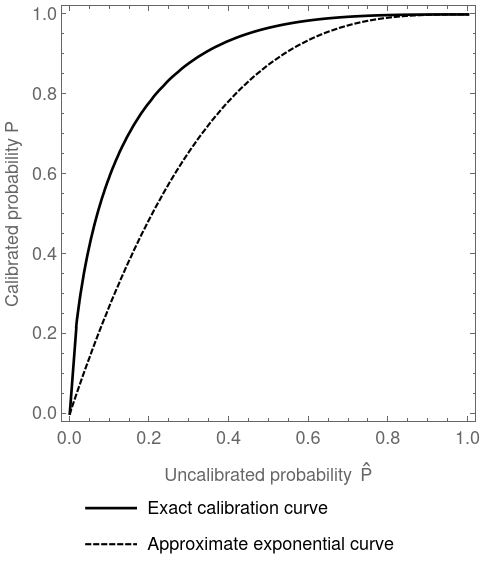}
		\caption{Calibration curve in the fully random case. The black solid line is the exact calibration curve \eqref{eq:exactcal} while the thin dashed line is its exponential approximate form. We use $K=3$.}
		\label{fig:justcal}
	\end{figure}
	
	Then, in the fully random case, the calibration map is the incomplete gamma function $c(\hat{P}) = \frac{\Gamma(K,\log 1/\hat{P})}{(K-1)!}$. Now, we look for an approximation of this form in terms of elementary functions. The resulting calibration curve is shown in Figure \ref{fig:justcal}. To this end, we expand the incomplete Gamma function around $\hat{P}=1$ resulting in
	$c(\hat{P}) \sim 1 - \frac{(1-\hat{P})^K}{(K-1)!}$ and then ensure that the approximation is properly normalized $P \in (0,1)$. The simplified calibration curve reads $c_{\textrm{appr}}(\hat{P}) = 1 - (1-\hat{P})^K$. We show both exact and approximate curves in Figure \ref{fig:justcal}. 
	
	This simple argument justifies the exponential form of the calibration curve $c_{\textrm{appr}}(\hat{P}) \sim \hat{P}^\alpha$. At the same time, it is based on a generic example and needs further generalization resulting in formula \eqref{eq:calform} where we make a linear combination of exponential terms and train both the linear coefficients $\beta$ as well as the exponents $\alpha$.
	
	\section{Calibration stability}
	\label{appendix:calstability}
	The exponential calibration proposed in this work is an integral and novel part of the \our{} approach. Therefore, to assess any shortcomings of the method, we inspect it under changes in the number of adjustable parameters $M$ and the size of the calibration dataset $n_b$. 
	
	\begin{figure}[h!]
		\centering
		\includegraphics[width=1\textwidth]{ 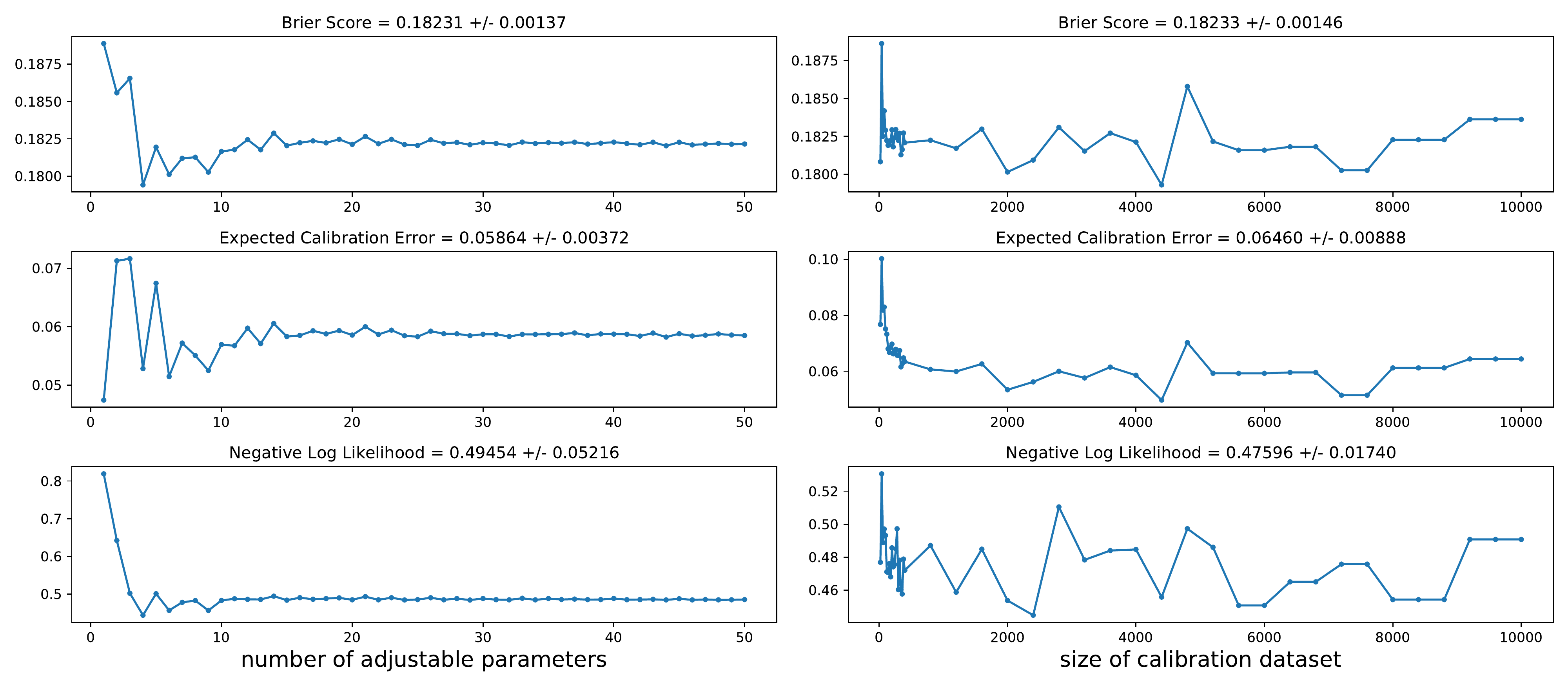}
		\caption{Stability of the calibration method under changes in the number of adjustable parameters $M$ and in the size of the calibration dataset $n_b$. We inspect variability of the NLL, ECE, Brier score as we change $M \in (1,50)$ and $n_b \in (20,10000)$ and identify a stable region of parameters for $M>10$ and $n_b>400$.}
		\label{fig:stability}
	\end{figure}
	
	The number of parameters $M$ refers to the pairs of weights $(\alpha_i,\beta_i)$ where $i=1,\ldots,M$ as defined in \eqref{eq:calibr}. The size of the calibration dataset $n_b$ is in turn discussed in Section \ref{sec:expcal}. We assess the robustness of the calibration algorithm in Figure \ref{fig:stability} by calculating ECE, NLL, and Brier score metrics across changes in $M \in (1,50)$ and $n_b \in (20,10000)$. 
	
	In the case of adjustable parameters $M$, after a brief instability period when $M<10$, we enter the regime with $M>10$ where the metric values do not change appreciably. Similarly, a similar stability region is found if the size of the dataset $n_b$ is greater than $400$. In conclusion, the calibration method is robust with respect to these parameter changes.
	
	\section{Table of metrics for the dataset shift experiment}
	\label{appendix:resultstab}
	The tables below report quartiles of Brier score, negative log-likelihood, and ECE in the dataset shift experiment considered in Section \ref{sec:expshift} where quartiles are computed over all corrupted variants of the dataset.
	\begin{center}
		\begin{tabular}{l|rrrrr}
			\toprule
			method & SLOVA &  Dropout &  Ensemble &  LL Dropout &  LL SVI \\
			\midrule
			Brier Score (25th) &  0.215 &    0.217 &     0.160 &       0.258 &   0.241 \\
			Brier Score (50th) &  0.299 &    0.365 &     0.294 &       0.412 &   0.405 \\
			Brier Score (75th) &  0.553 &    0.657 &     0.556 &       0.746 &   0.735 \\
			\hline
			ECE (25th)         &  0.021 &    0.012 &     0.018 &       0.089 &   0.075 \\
			ECE (50th)         &  0.027 &    0.039 &     0.024 &       0.151 &   0.137 \\
			ECE (75th)         &  0.103 &    0.211 &     0.107 &       0.313 &   0.291 \\
			\hline
			NLL (25th)         &  0.507 &    0.452 &     0.334 &       0.619 &   0.572 \\
			NLL (50th)         &  0.709 &    0.811 &     0.636 &       1.070 &   1.062 \\
			NLL (75th)         &  1.383 &    1.743 &     1.529 &       2.377 &   2.368 \\
			\bottomrule
		\end{tabular}
		\begin{tabular}{l|rrrr}
			\toprule
			method &   SVI &  Temp Scaling &  Softmax &  OVA DM \\
			\midrule
			Brier Score (25th) & 0.245 &         0.228 &    0.242 &   0.300 \\
			Brier Score (50th) & 0.371 &         0.382 &    0.412 &   0.398 \\
			Brier Score (75th) & 0.635 &         0.649 &    0.733 &   0.702 \\
			\hline
			ECE (25th)         & 0.034 &         0.014 &    0.074 &   0.063 \\
			ECE (50th)         & 0.074 &         0.049 &    0.133 &   0.079 \\
			ECE (75th)         & 0.214 &         0.140 &    0.284 &   0.166 \\
			\hline
			NLL (25th)         & 0.520 &         0.475 &    0.574 &   0.534 \\
			NLL (50th)         & 0.839 &         0.850 &    1.012 &   0.731 \\
			NLL (75th)         & 1.714 &         1.589 &    2.219 &   1.294 \\
			\bottomrule
		\end{tabular}
	\end{center}
\end{document}